\pdfoutput=1
\documentclass[conference]{IEEEtran}
\usepackage{booktabs}    
\usepackage{multirow}    
\usepackage{array}       
\usepackage{graphicx}    
\usepackage{amsmath}     

\usepackage{amsthm}
\theoremstyle{remark}
\newtheorem{remark}{Remark}

\newtheorem{theorem}{Theorem}
\usepackage{algorithmicx}
\usepackage{amssymb}  
\usepackage{amsfonts} 
\usepackage{algorithm}
\usepackage{algpseudocode} 
\usepackage{nicematrix} 
\begin{document}

\title{TS-ACL: Closed-Form Solution for Time Series-oriented  Continual Learning}


\author{
Jiaxu Li$^{1}$, 
Kejia Fan$^{1}$, 
Songning Lai$^{2}$, 
Linpu Lv$^{3}$,
Jinfeng Xu$^{4}$,
Jianheng Tang$^{5}$, \\
Anfeng Liu$^{1}$, 
Houbing Herbert Song$^{6}$, 
Yutao Yue$^{2}$, 
Yuanhuai Liu$^{5}$,
Huiping Zhuang$^{7}$\\
$^1$Central South University, $^2$The Hong Kong University of Science and Technology (Guangzhou), \\
$^3$Zhengzhou University,
$^4$The University of Hong Kong,
$^5$Peking University,
\\$^6$University of Maryland,
$^3$South China University of Technology
}

\maketitle
\begin{abstract}
Time series classification underpins critical applications such as healthcare diagnostics and gesture-driven interactive systems in multimedia scenarios. However, time series class-incremental learning (TSCIL) faces two major challenges: catastrophic forgetting and intra-class variations. Catastrophic forgetting occurs because gradient-based parameter update strategies inevitably erase past knowledge. And unlike images, time series data exhibits subject-specific patterns, also known as intra-class variations, which refer to differences in patterns observed within the same class. While exemplar-based methods fail to cover diverse variation with limited samples, existing exemplar-free methods lack explicit mechanisms to handle intra-class variations. To address these two challenges, we propose TS-ACL, which leverages a gradient-free closed-form solution to avoid the catastrophic forgetting problem inherent in gradient-based optimization methods while simultaneously learning global distributions to resolve intra-class variations. Additionally, it provides privacy protection and efficiency. Extensive experiments on five benchmark datasets covering various sensor modalities and tasks demonstrate that TS-ACL achieves performance close to joint training on four datasets, outperforming existing methods and establishing a new state-of-the-art (SOTA) for TSCIL.
\end{abstract}

\begin{IEEEkeywords}
Time series classification, Continual learning
\end{IEEEkeywords}

\section{Introduction}

With significant research attention in recent years, pattern recognition in time series plays a critical role in various applications, such as healthcare diagnostics \cite{healthcare}, industrial production \cite{industry}, and urban computing \cite{urban}.
Deep Learning (DL) approaches have gained widespread popularity due to their superior performance, and they typically rely on offline, static datasets that assume data to be independently and identically distributed (i.i.d.) \cite{CLPOS}.

However, as shown in Figure 1, real-world scenarios often involve continual data streams from sensors, resulting in an ever-growing volume of time series data with incremental classes, where the i.i.d. assumption no longer holds \cite{qiao2024class}.
Additionally, as new data with previously unseen classes may emerge over time, the dataset becomes more complex \cite{DT2W}.
This dynamic environment requires the DL models to continually adapt and learn from new data samples.
Unfortunately, this process is often hindered by the well-known problem of catastrophic forgetting, in which the models tend to forget previously learned knowledge when exposed to new data samples \cite{CIL_forgetting_1, CIL_survey}.

Numerous Class-Incremental Learning (CIL) methods have been proposed.
These methods can be broadly categorized into two types: exemplar-based methods \cite{ICARL, ER, FASTICARL} and exemplar-free methods \cite{LwF, MAS, DT2W}. 
exemplar-based methods store historical samples or use generative models to produce synthetic samples, called the exemplar, to enable the model to retain previously learned knowledge \cite{ICARL, ER, GR}.
While these exemplar-based methods can achieve impressive performance, they inherently raise significant concerns about privacy due to the need to store historical data \cite{ACL_1}.
More importantly, in many practical applications, particularly edge computing scenarios, the limited storage and computational resources usually render such methods impractical \cite{edge}.
On the other hand, exemplar-free methods typically aim to protect learned knowledge by modifying the loss function or optimizing the learning process without exemplar \cite{LwF, CIL_survey}.
Although these methods avoid privacy issues, they still suffer from significantly decreased accuracy with suboptimal performance, limiting their practical utility \cite{qiao2024class}.

\begin{figure}[tbp]
  \centering
  \includegraphics[width=1\linewidth]{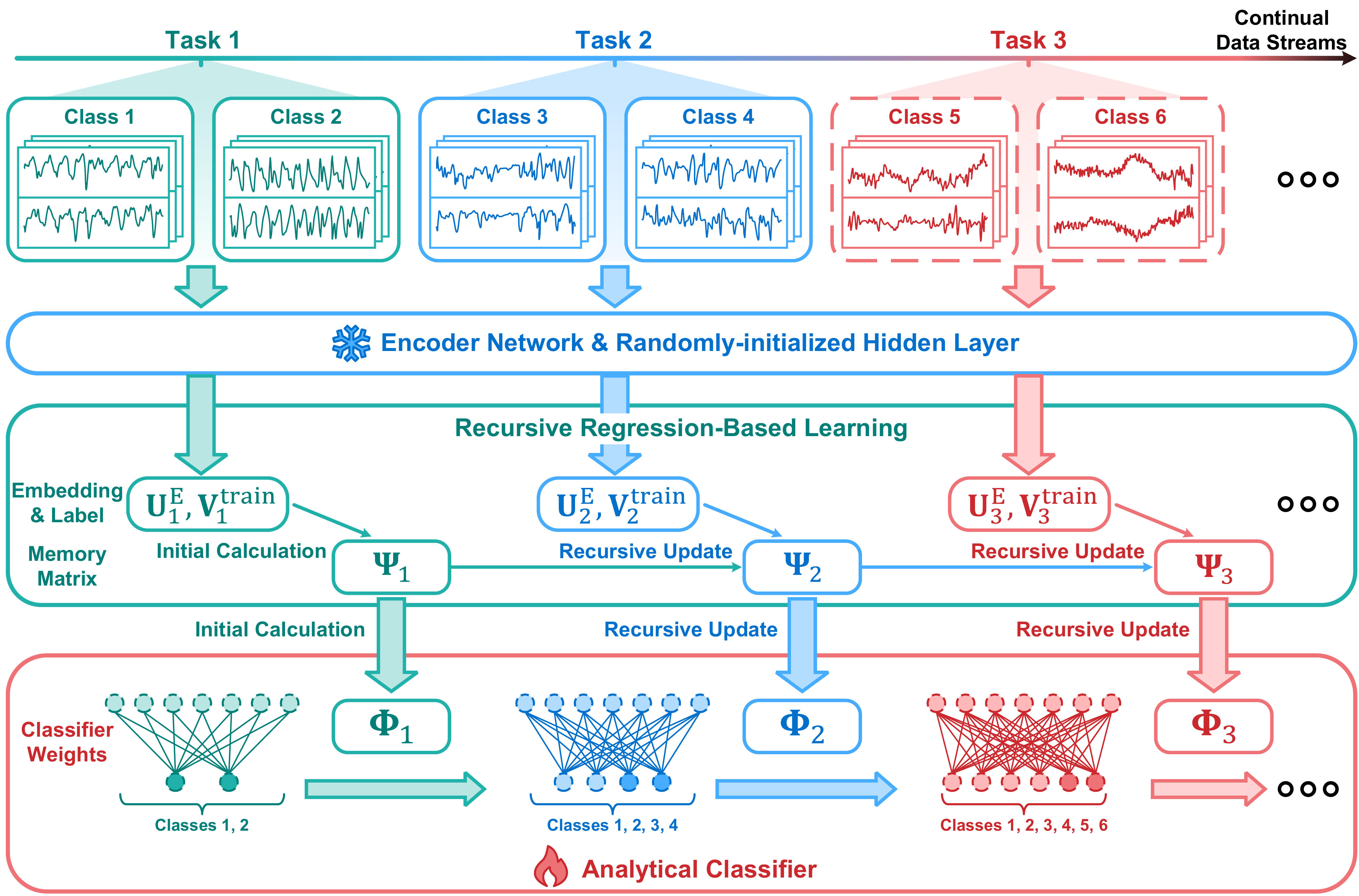}
  \caption{Schematic diagram of Time Series Class-Incremental Learning (TSCIL) and a brief TS-ACL model update illustration. As the data stream continuously arrives, the model learns new classes in each task.}
  \label{fig:1}
\end{figure}

Furthermore, a more fundamental and general limitation within current CIL methods, whether exemplar-based or exemplar-free, lies in their reliance on gradient-based backpropagation techniques.
Since gradients computed from new data can conflict with or even contradict those derived from historical data, current CIL methods fail to fully address the problem of catastrophic forgetting.
Because updates based on conflicting gradients inevitably lead to the erasure of previously acquired knowledge \cite{CIL_forgetting_1}. Additionally, intra-class variations in time series data present a significant challenge. Time series belonging to the same class may exhibit markedly different patterns depending on the individual. Addressing intra-class variations is essential for time series class-incremental learning (TSCIL) \cite{qiao2024class}. To address these two challenges, we introduce TS-ACL, as shown in Figure \ref{fig:1}.

We identify gradient descent as the root cause of catastrophic forgetting and transform each model update into a gradient-free analytical learning process with a closed-form solution. As a result, TS-ACL completely eliminates the need for gradient-based updates, fundamentally addressing the issue of catastrophic forgetting (stability), while the global distribution learned through closed-form solutions also resolves the problem of intra-class variations in time series incremental learning. By leveraging a pre-trained frozen encoder for feature extraction, TS-ACL only requires lightweight recursive updates to the analytical classifier for each task. Furthermore, to enhance the accuracy and robustness of the model when handling time series data, our approach introduces a multi-scale feature fusion technique and randomly maps features into a high-dimensional space. This method integrates features from different layers, enabling a more comprehensive capture of the complex patterns inherent in time series data, thereby improving model performance (plasticity). In this way, TS-ACL simultaneously achieves stability, plasticity, privacy protection, and lightweight resource consumption.

The key contributions of this work are as follows: 
\begin{enumerate} 
\item We propose TS-ACL, which leverages the closed-form solution through recursive ridge regression, addressing the issues of catastrophic forgetting and intra-class variations in TSCIL.

\item We then enhance the separability of time series classes by designing multi-scale feature fusion and mapping them into a high-dimensional space.

\item Extensive experiments demonstrate that TS-ACL achieves state-of-the-art performance across five benchmark datasets outperforming current methods.
\end{enumerate}

\begin{figure*}[htbp]
  \centering
  \includegraphics[width=1.0\linewidth]{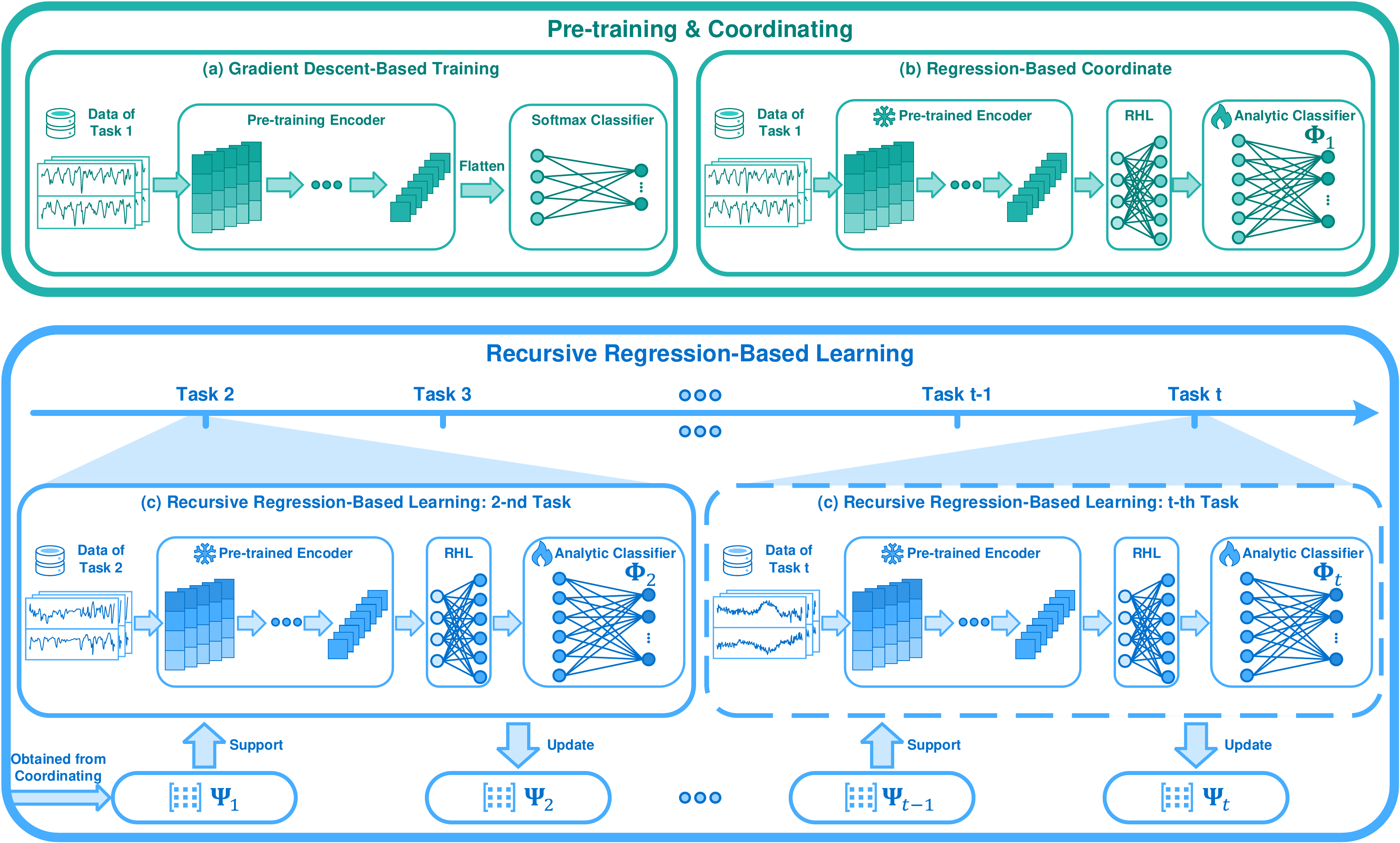}
  \caption{Overview of the proposed TS-ACL. In (a), a pre-trained encoder is first obtained on Task 1 through Gradient Descent-Based Training. Next, as shown in (b), a Regression-Based Coordinate module with an RHL is introduced before the classification head to enhance the feature dimension, resulting in $\mathbf{\Psi_1}$ from Task 1 data. Finally, in (c), a Recursive Regression-Based Learning process is applied  across tasks.}
  \label{fig:overview}
\end{figure*}

\section{Related Work}

\noindent\textbf{Time Series Classification.}
Time series classification aims to assign categorical labels to time series data based on its patterns or characteristics, playing an increasingly important role in numerous domains \cite{healthcare, industry, TSC_survey_1}.
In the early stages of research, the focus was primarily on non-DL methods, which mainly included distance-based methods \cite{TSC_distance_1, TSC_distance_2} and ensembling methods \cite{TSC_ensembling_1, TSC_ensembling_2}.
The distance-based methods rely on various time series distance measurement techniques, with the Dynamic Time Warping as a representative example \cite{DTW}, using 1-nearest neighbor classifiers.
The ensembling methods integrate multiple individual 1-nearest neighbor classifiers with different distance metrics to achieve improved performance \cite{TSC_ensembling_1, TSC_survey_1}.
These non-DL methods typically incur significant computational overhead and struggle with scalability when applied to large-scale datasets \cite{TSC_survey_1, TSC_survey_2}.

Recently, DL-based methods have been extensively studied and have achieved encouraging performance \cite{TSC_survey_2, DT2W, TSC_DL_1}. 
These methods are typically trained on offline, static, and i.i.d. datasets, with little consideration given to scenarios involving continual data perception in real-world environments \cite{DT2W}.
In practice, new time series data samples are continually collected over time, necessitating incremental model updates \cite{DT2W, CIL_survey}.
Unfortunately, classical DL-based methods inevitably face the problem of catastrophic forgetting, where the models forget previously learned knowledge when exposed to new data samples \cite{CIL_forgetting_1, CIL_survey}.

\noindent\textbf{Class-incremental Learning.}
To alleviate the problem of catastrophic forgetting, numerous studies on CIL \cite{CIL_survey, LwF, FASTICARL} have emerged and can be broadly categorized into two main types: exemplar-based methods \cite{ICARL, ER} and exemplar-free methods \cite{LwF, MAS, DT2W}.
exemplar-based methods store a subset of historical samples, called the exemplar, and replay them during incremental training to alleviate the problem of catastrophic forgetting \cite{ER, GR}.
Meanwhile, exemplar-free methods attempt to preserve prior knowledge by incorporating additional terms into the loss function or by explicitly designing and manipulating the optimization process \cite{MAS, DT2W}.

In practice, exemplar-based methods typically achieve better performance thanks to the exemplar, but they inherently raise significant concerns about privacy \cite{ACL_1}.
In many practical applications, particularly edge computing scenarios, the limited storage and computational resources usually render such methods impractical \cite{edge}.
Meanwhile, although exemplar-free methods can avoid privacy issues, they typically suffer from significantly suboptimal performance \cite{ACL_2, CIL_survey}.
More critically, both these methods depend on gradient-based updating techniques, which fundamentally do not solve the problem of catastrophic forgetting \cite{CIL_forgetting_1}.

\noindent\textbf{Time Series Class-Incremental Learning.}
DT2W \cite{DT2W}, CLOPS \cite{CLPOS}, and FastICARL \cite{FASTICARL} are recently proposed methods for handling time series class-incremental learning (TSCIL). DT2W is a knowledge distillation-based approach that uses Soft-DTW (Dynamic Time Warping) to align feature maps, making it particularly suitable for multivariate time series classification. It appears to effectively balance the learning of new and old knowledge, reducing forgetting. CLOPS is designed for cardiac arrhythmia diagnosis, employing an experience replay strategy that manages memory buffers through importance storage and uncertainty retrieval. FastICARL is a fast variant of iCaRL \cite{ICARL} which enhances learning efficiency by replacing the selection strategy with KNN and quantizing compressed memory samples. However, these methods face the following challenges: 1) Data privacy: In sensitive fields like healthcare, storing raw time series data may not be feasible. 2) Intra-class variations: Time series data often exhibit significant intra-class variations, requiring methods to effectively manage these variations while learning new classes. 3) Efficiency: Due to the sequential nature of time series data, efficient online learning algorithms are crucial for real-time processing of data streams. We recognize these challenges and propose TS-ACL, which employs a closed-form solution for continual learning, perfectly addressing the aforementioned challenges and establishing a new SOTA in TSCIL.

\section{Motivation}
\subsection{Preliminaries}
Let the model be continually trained for $T$ tasks, where the training data for each task comes from different classes. Let $\mathcal{D}^{\text{train}}_t \sim \{ \mathbf{U}^{\text{train}}_t, \mathbf{V}^{\text{train}}_t \}$ and $\mathcal{D}^{\text{test}}_t \sim \{ \mathbf{U}^{\text{test}}_t, \mathbf{V}^{\text{test}}_t \}$ denote the training and testing datasets at task $t$ ($t = 1, \dots, T$). Specifically, $\mathbf{U}_t^{\text{train}} \in \mathbb{R}^{N_t \times c \times l}$ (e.g., $N_t$ time series with a shape of $c \times l$) and $\mathbf{V}^{\text{train}}_t \in \mathbb{R}^{N_t \times d_{y_t}}$ (with task $t$ containing $d_{y_t}$ classes) represent the stacked input and label tensors. For a feature extractor $ f() $ with $ K $ layers, given any sample $ u \in \mathcal{U} $, the feature extracted by the $ k $-th layer of $ f() $ is denoted as $ f_k(u) \in \mathbb{R}^{c_k \times l_k} $. Correspondingly, the feature from the final layer is denoted as $ f_K(u) \in \mathbb{R}^{c_K \times l_K} $.

In the CIL scenario, the classes learned across different tasks are independent and non-overlapping \cite{CIL_survey, qiao2024class}. For each task $t$, the network learns from $\mathcal{D}^{\text{train}}_t$, where the classes differ from those in previous tasks. In task $t$, the data labels $\mathbf{V}^{\text{train}}_t$  belong to a task-specific class set $\mathcal{C}_t$, disjoint from other sets $\mathcal{C}_{t'}$ (for $t' \neq t$). Thus, data from different tasks exhibit distinct class distributions. For each task $t$, the objective is to use previous parameters $\Theta_{t-1}$ and $\mathcal{D}^{\text{train}}_t$ to update the model to $\Theta_t$, ensuring both stability (retaining past knowledge) and plasticity (learning new knowledge).

\subsection{Addressing Catastrophic Forgetting}
The first phase of TS-ACL involves training an encoder network $f(\cdot)$ on the initial task using gradient descent optimization. A key insight of our approach is that catastrophic forgetting primarily occurs due to conflicting gradients between sequential tasks. When the data distributions between tasks are inconsistent, gradient updates from new tasks can erase knowledge learned from previous tasks \cite{CIL_forgetting_1}.

To fundamentally address this issue, we propose to:
1) Train the encoder only on the initial task to learn general feature extraction capabilities
2) Freeze the encoder weights for all subsequent tasks
3) Adopt a gradient-free approach through recursive ridge regression for incremental learning

This design eliminates the root cause of catastrophic forgetting by avoiding gradient updates after the initial training. The frozen encoder maintains stable feature extraction, while ridge regression enables analytical updates that do not interfere with previously learned knowledge.

Specifically, we train the encoder in task 1 using cross-entropy loss:
\begin{equation}
\min_{\Theta_{\text{encoder}}} \mathcal{L}_{\text{CE}}(f(\mathbf{U}^{\text{train}}_1; \Theta_{\text{encoder}}), \mathbf{V}^{\text{train}}_1).
\end{equation}
After training, we freeze the encoder weights for all subsequent tasks, ensuring stable feature extraction.

\subsection{Handling Intra-Class Variations}

Time series data presents unique challenges in handling intra-class variations \cite{qiao2024class}. As shown in Figure~\ref{fig:tsne}, the same class (such as the same activity) exhibits distinctly different temporal patterns when performed by different subjects, leading to clear sub-clusters within classes.

Consider a simple scenario: learning two consecutive tasks from the same group of subjects $S = \{s_1, \ldots, s_k\}$ with prior distribution $p(S) = \sum_{s \in S} p(S = s)$. The input distribution conditioned on subject $s$ can be expressed as $p(\mathbf{U}) = \sum_{s \in S} p(\mathbf{U}|S = s)p(S = s)$, where $\mathbf{U}$ represents the time series. For two tasks with incremental learning distributions $p(\mathbf{U}_1, \mathbf{V}_1)$ and $p(\mathbf{U}_2, \mathbf{V}_2)$, the learning objective is:

\begin{align}
\max & \sum_{s \in S} p_1(s) \sum_{(\mathbf{u}_1, \mathbf{v}_1) \in \mathcal{D}_1^{train}} \log p(\mathbf{v}_1|\mathbf{u}_1, s) \nonumber \\ 
 + &\sum_{s \in S} p_2(s) \sum_{(\mathbf{u}_2, \mathbf{v}_2) \in \mathcal{D}_2^{train}} \log p(\mathbf{v}_2|\mathbf{u}_2, s).
\end{align}

Traditional exemplar-based methods (such as ER) often fail to maintain $p_1(s) = p(s)$ when managing memory samples, leading to distribution shift in replay data. This is particularly problematic for time series data, as stored samples may not adequately represent the complete temporal variations and subject-specific patterns within each class. Meanwhile, exemplar-free methods in continual learning for time-series data attempt to preserve knowledge by constraining model parameters, but they lack effective mechanisms to handle intra-class variations.

\begin{figure}[htbp]
\centering
\includegraphics[width=0.4\textwidth]{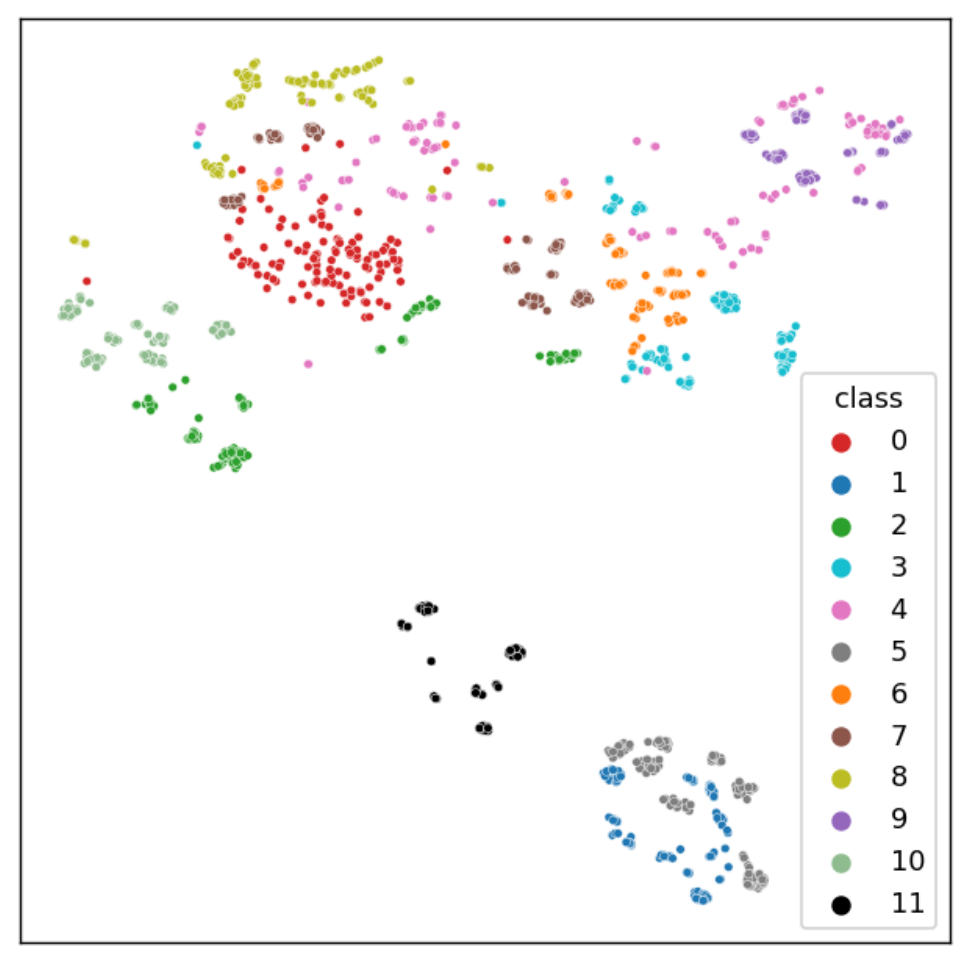} 
\caption{t-SNE visualization of DSA dataset features. Different colors represent different classes, and it can be clearly seen that a single class contains many sub-clusters.}
\label{fig:tsne}
\end{figure}

Now consider the standard ridge regression problem, with the objective function:
\begin{equation}
\min_{\mathbf{\Phi}} \| \mathbf{V^{\text{train}}} - \mathbf{U}^{\text{feat(K)}} \mathbf{\Phi} \|_F^2 + \gamma \| \mathbf{\Phi} \|_F^2,
\end{equation}
where $\mathbf{U}^{\text{feat(K)}}$ and $\mathbf{V^{\text{train}}}$ represent the complete feature matrix and stacked one-hot label matrix, respectively. This method can learn the complete distribution of the data, effectively addressing the intra-class variations in time series data. However, in the context of continual learning, accessing all historical data is not feasible. Therefore, we designed a recursive ridge regression framework that achieves effects equivalent to joint training through closed-form solutions, learning the complete data distribution at once and achieving the same effect as standard ridge regression. Below we detail our recursive ridge regression framework.

\section{Method}
\subsection{Regression-Based Coordinate}

First of all, we introduce ridge regression learning, which is divided into three steps.
The first step involves extracting the feature matrix (denoted as $ \mathbf{U}^{\text{feat(K)}}_1  $) by passing the input tensor $ \mathbf{U}^{\text{train}}_1 $ through the trained encoder network, followed by an average pooling operation, i.e.,
\begin{equation}
\mathbf{U}^{\text{feat(K)}}_1 = \text{AvgPool}(f(\mathbf{U}^{\text{train}}_1)),
\label{eq:2}
\end{equation}
where $ \mathbf{U}^{\text{feat(K)}}_1 \in \mathbb{R}^{N_1 \times c_{K}} $.

However, relying solely on the features from the final layer may be insufficient. Therefore, we further incorporate more general features from the preceding layers by establishing a skip connection. Correspondingly, we represent the concatenated features as 
\begin{equation}
\mathbf{U}^{\text{stack}}_1 = \text{Concat}(\mathbf{U}^{\text{feat(1)}}_1, \cdots, \mathbf{U}^{\text{feat(K-1)}}_1, \mathbf{U}^{\text{feat(K)}}_1).
\end{equation}

\begin{remark}
Building upon our understanding of multi-scale temporal patterns, our feature concatenation strategy integrates features from different network layers to capture complementary characteristics.
\end{remark}

    Next, instead of directly mapping the features to the classification output via a single classifier layer, we add a Randomly-initialized Hidden Layer (RHL) and use ReLU as activation. Specifically, the feature $ \mathbf{U}^{\text{stack}}_1 $ is expanded into $ \mathbf{U}_{\text{E}} $ as follows:
    \begin{equation}
    \mathbf{U}^{\text{E}}_1 = \text{ReLU}(\mathbf{U}^{\text{stack}}_1 \mathbf{\Phi}^{\text{E}}),
    \label{eq:3}
    \end{equation}
    where $ \mathbf{U}^{\text{E}}_1 \in \mathbb{R}^{N_1 \times d_{\text{E}}} $, with $ d_{\text{E}} $ representing the feature expansion size (generally $ d_{\text{stack}} < d_{\text{E}} $). Here, $ \mathbf{\Phi}^{\text{E}} $ is the matrix used to expand the features extracted from the encoder. The matrix $ \mathbf{\Phi}^{\text{E}} $ is initialized by sampling its elements from a normal distribution.
     
    \begin{remark}
    The necessity of the RHL layer lies in the fact that nonlinearly mapping features to a high-dimensional space can increase the probability of features being linearly separable \cite{RHL}.
    \end{remark}
    To further enhance the performance of our method, we employ a straightforward ensemble learning approach, which we refer to as TS-ACL-E. We initialize the RHL layer with different random seeds. Let $ n $ denote the number of distinct RHL layers. We transform the output of each model into probabilities using the softmax function, then aggregate these probabilities, and finally predict the class with the highest cumulative probability. This process can be formally expressed as:
    
    \begin{equation}
    \hat{v} = \arg\max_m\frac{1}{n}\sum_{i=1}^n P_i(v=m\mid u),
    \end{equation}
    where $\hat{v}$ represents the predicted class, $P_i(v=m\mid u)$ denotes the probability of class $m$ given input $u$ from the $i^{th}$ model, and $n$ is the total number of models with different initializations.
    
    Finally, the expanded embeddings $ \mathbf{U}^{\text{E}}_1 $ are mapped to the stacked one-hot label matrix $ \mathbf{V}_1^{\text{train}} $ using a ridge regression, by solving the following optimization problem:
    \begin{equation}
    \arg\min_{\mathbf{\Phi}_{1}} (\| \mathbf{V}_1^{\text{train}} - \mathbf{U}^{\text{E}}_1 \mathbf{\Phi}_{1} \|_\text{F}^2 + \gamma \| \mathbf{\Phi}_{1} \|_\text{F}^2),
    \end{equation}
    where $ \gamma $ is a regularization parameter and $\left\lVert\cdot\right\lVert_{\text{F}}$ is the Frobenius norm. The solution to this problem is given by:
    \begin{equation}
    \begin{aligned}
    \widehat{\mathbf{\Phi}}_{1} = \left( \mathbf{U}^{(\text{E})\top}_1 \mathbf{U}^{\text{E}}_1 + \gamma \mathbf{I} \right)^{-1} \mathbf{U}^{(\text{E})\top}_1 \mathbf{V}_1^{\text{train}},
    \label{eq:5}
    \end{aligned}
    \end{equation}
    where the notation $^\top$ is the transpose operation and the notation $^{-1}$ represents the inverse operation.

    \subsection{Recursive Regression-Based Learning}
    The previous section introduced the process of ridge regression learning, which, however, is not suitable for continual learning. In the following, we present recursive ridge regression learning adapted for continual learning as shown in Figure \ref{fig:overview}. To demonstrate this, without loss of generality, assume we are given $ \mathcal{D}^{\text{train}}_{1}, \dots, \mathcal{D}^{\text{train}}_{t-1} $, and let
    \begin{equation}
    \begin{aligned}
    \mathbf{U}_{1:t-1}^{\text{E}} \in \mathbb{R}^{N_{1:t-1} \times d_{\text{E}}}, \quad \mathbf{V}_{1:t-1}^\text{train} \in \mathbb{R}^{N_{1:t-1} \times \sum_{i=1}^{t-1} d_{y_i}}
    \end{aligned}
    \end{equation}
    represent the accumulated embeddings and label tensors, respectively, from task 1 to $ t-1 $. Here, $ N_{1:t-1} $ indicates the total number of data samples from task 1 to $ t-1 $. The sparse structure of $ \mathbf{V}_{1:t-1}^\text{train} $ arises because of the mutually exclusive classes across tasks. The learning problem can then be formulated as follows:
    \begin{equation}
    \begin{aligned}
    \arg \min_{\mathbf{\Phi}_{t-1}} \left\| \mathbf{V}_{1:t-1}^\text{train} - \mathbf{U}_{1:t-1}^{\text{E}} \mathbf{\Phi}_{t-1} \right\|_F^2 + \gamma \left\| \mathbf{\Phi}_{t-1} \right\|_F^2,
    \end{aligned}
    \end{equation}
    according to \ref{eq:5}, at task $ t-1 $, we have:
    \begin{equation}
    \begin{aligned}
    \hat{\mathbf{\Phi}}_{t-1} = \left( \mathbf{U}_{1:t-1}^{(\text{E})\top} \mathbf{U}_{1:t-1}^{\text{E}} + \gamma \mathbf{I} \right)^{-1} \mathbf{U}_{1:t-1}^{(\text{E})\top} \mathbf{V}_{1:t-1}^\text{train},
    \end{aligned}
    \label{eq:9}
    \end{equation}
    where $ \hat{\mathbf{\Phi}}_{t-1} \in \mathbb{R}^{d_{\text{E}} \times \sum_{i=1}^{t-1} d_{y_i}} $, with the column size expanding as $ t $ increases. Let
    \begin{equation}
    \begin{aligned}
    \mathbf{\Psi}_{t-1} = \left( \mathbf{U}_{1:t-1}^{(\text{E})\top} \mathbf{U}_{1:t-1}^{\text{E}} + \gamma \mathbf{I} \right)^{-1}
    \end{aligned}
    \label{eq:10}
    \end{equation}
    denote the aggregated temporal inverse correlation matrix, which captures the correlation information from both current and past samples. Based on this, our goal is to compute $ \hat{\mathbf{\Phi}}_{t} $ using only $ \hat{\mathbf{\Phi}}_{t-1} $, $ \mathbf{\Psi}_{t-1} $, and the current task's data $ \mathbf{U}^{\text{train}}_t $, without involving historical samples such as $ \mathbf{U}_{1:t-1}^{\text{E}} $. The process is formulated in the following theorem.

    \begin{theorem}
    
    \label{theorem:main}
    The $\mathbf{\Phi}$ weights, recursively obtained by
    \begin{align}
    \hat{\mathbf{\Phi}}_t = \left[ \hat{\mathbf{\Phi}}_{t-1} - \mathbf{\Psi}_{t} \mathbf{U}_{t}^{(\text{E})\top} \mathbf{U}_{t}^{\text{E}} \hat{\mathbf{\Phi}}_{t-1} \quad \mathbf{\Psi}_t \mathbf{U}_{t}^{(\text{E})\top} \mathbf{V}^{\text{train}}_t \right]
    \label{eq:11}
    \end{align}
    are equivalent to those obtained from Eq.(\ref{eq:9}) for task $ t $. The matrix $ \mathbf{\Psi}_t $ can also be recursively updated by
    \begin{equation}
    \mathbf{\Psi}_t = \mathbf{\Psi}_{t-1} - \mathbf{\Psi}_{t-1} \mathbf{U}_{t}^{(\text{E})\top} \left( \mathbf{I} + \mathbf{U}_{t}^{\text{E}} \mathbf{\Psi}_{t-1} \mathbf{U}_{t}^{(\text{E})\top} \right)^{-1} \mathbf{U}_{t}^{\text{E}} \mathbf{\Psi}_{t-1}.
    \label{eq:12}
    \end{equation}
    \end{theorem}
    
    \begin{proof}
     See the supplementary materials.
    \end{proof}
           
    \label{AL-based class incremental learning}
    
    \subsection{Theoretical Analysis}
    \noindent\textbf{Privacy Protection.}  TS-ACL ensures data privacy in two ways: first, by eliminating the need to store historical data samples; second, by guaranteeing that historical raw data samples cannot be recovered from the $\mathbf{\Psi}$ matrix through reverse engineering.

    \noindent\textbf{Time Complexity Analysis.} In the $ t $-th incremental learning task, 
    the time complexity for updating the matrix $\mathbf{\Psi_t}$ is $\mathcal{O}(N_t^3+N_td_\text{E}^2)$, 
    while the time complexity for updating the weight matrix $\hat{\mathbf{\Phi_t}}$ is 
    $\mathcal{O}(d_\text{E}^2\sum_{i=1}^{t} d_{y_i}+
    N_td_\text{E}^2)$. Therefore, 
    the overall time complexity is $\mathcal{O}(N_t^3+d_\text{E}^2\sum_{i=1}^{t} d_{y_i}+N_td_\text{E}^2)$. 

    \noindent\textbf{Space Complexity Analysis.}
    In terms of space complexity, the $ t $-th task requires storing the matrix 
    $\mathbf{\Psi_t} \in \mathbb{R}^{d_{\text{E}} \times d_{\text{E}}}$ and the weight matrix 
    $\hat{\mathbf{\Phi_t}} \in \mathbb{R}^{d_{\text{E}} \times \sum_{i=1}^{t} d_{y_i}}$, resulting in an overall space complexity of $\mathcal{O}(d_{\text{E}}^2+d_{\text{E}}\sum_{i=1}^{t} d_{y_i})$.
    
   \label{efficiency}
    
    \section{Experiment}

    \subsection{Experiment Setup}

    \noindent\textbf{Datasets.}
    We continue previous research \cite{qiao2024class} by selecting five balanced time series datasets, each containing samples with the same length and variables, which is shown in Table \ref{tabledataset}.

\noindent\textbf{UCI-HAR} includes temporal sequences obtained from smartphone inertial sensors during the execution of six daily activities. Data were collected at a frequency of 50Hz from 30 participants of varying ages. Each input sequence consists of nine channels, covering a temporal span of 128 timesteps.

\noindent\textbf{UWave} provides over 4000 samples collected from eight individuals performing eight distinct gesture patterns. The data consists of time series captured along three accelerometer axes, with each sample having a dimensionality of 3 and spanning 315 timesteps.

\noindent\textbf{DSA} gathers motion sensor data from 19 different sports activities performed by eight volunteers. Each segment serves as a single sample, featuring data across 45 channels and 125 timesteps. To make classes be split equally, 18 activity classes were selected for experimentation.

\noindent\textbf{GRABMyo} is a large-scale Surface Electromyography (sEMG) dataset designed for hand gesture recognition. It includes signals corresponding to 16 gestures performed by 43 participants across three sessions. Recordings last five seconds, sampled from 28 channels at 2048 Hz. For experimentation, data from a single session were used and downsampled to 256 Hz. A non-overlapping sliding window of 0.5 seconds (128 timesteps) was applied to segment the signals into samples. The data were then split into training and testing sets in a 3:1 ratio, ensuring both sets include all subjects to mitigate distribution shifts across participants.

\noindent\textbf{WISDM} is a human activity recognition dataset that captures sensor data from 18 activities performed by 51 participants. Using the phone accelerometer modality, samples were generated with a non-overlapping sliding window of 200, representing 10-second time series collected at a 20Hz frequency. Similar to GRABMyo, the data were split into training and test sets at a 3:1 ratio, ensuring all participants are represented in both splits.

\noindent\textbf{CIL Setting.}
We train and evaluate on five datasets, each divided into $n$ tasks with two distinct classes per task, and randomly shuffle class order before partitioning. To evaluate the robustness, we use different random seeds to conduct five independent experiments and report the average results with standard deviation.

\noindent\textbf{Task Stream.}
Following the approach of \cite{qiao2024class}, we divide tasks into a validation stream and an experiment stream. In the UCI-HAR \cite{ismi2016k} and UWave datasets \cite{liu2009uwave}, which have fewer classes, the validation stream and experiment stream consist of 3 and 4 tasks, respectively. For other datasets with more classes, the validation stream includes 3 tasks, while the experiment stream covers the remaining tasks, as summarized in Table \ref{tabledataset}. We initially perform a grid search on the validation stream to optimize hyperparameters, then use the selected parameters to conduct experiments on the experiment stream.

        \begin{table}[htbp]
        \centering
        \caption{Summary of time series datasets used in the study, showing key characteristics including sensor channels (C), sequence length (L), data splits, class distribution, and experimental tasks.} 
        \label{tabledataset}
        \resizebox{0.45\textwidth}{!}{
        \begin{NiceTabular}{l|ccccc}
            \toprule
            Dataset & Shape $(C \times L)$ & Train Size & Test Size & \ Classes & \ Exp Tasks  \\ \midrule
            UCI-HAR \cite{ismi2016k} & $9\times128$ & 7352 & 2947 & 6 & 3 \\ 
            
            UWave \cite{liu2009uwave} & $3\times315$ & 896 & 3582 & 8 & 4\\
            
            DSA \cite{altun2010human} & $45\times125$ & 6840 & 2280 & 18 & 6\\ 
            
            GRABMyo \cite{pradhan2022multi} & $28\times128$ & 36120 & 12040 & 16 & 5\\ 
            
            WISDM \cite{weiss2019wisdm} & $3\times200$ & 18184 & 6062 & 18 & 6\\ 
            \bottomrule
      
        \end{NiceTabular}}
    \end{table}

    \begin{figure*}[htbp]
        \centering
            \includegraphics[width=\textwidth]{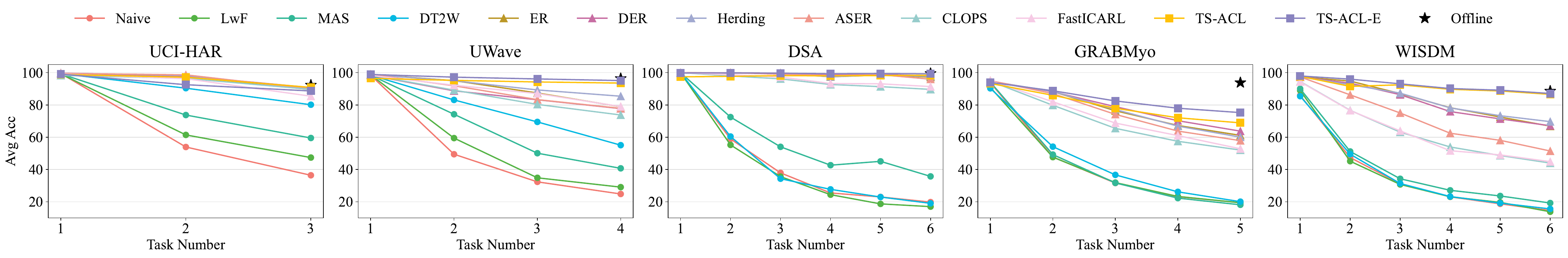} 
        \caption{The accuracy of all methods on five datasets changes as the tasks progress. "Naive" refers to the baseline method without any incremental learning strategies, while "Offline" indicates methods that are trained on the entire dataset at once.}
        \label{fig:mainfig_LN}
    \end{figure*}

        \begin{figure*}[htbp]
        \centering
            \includegraphics[width=\textwidth]{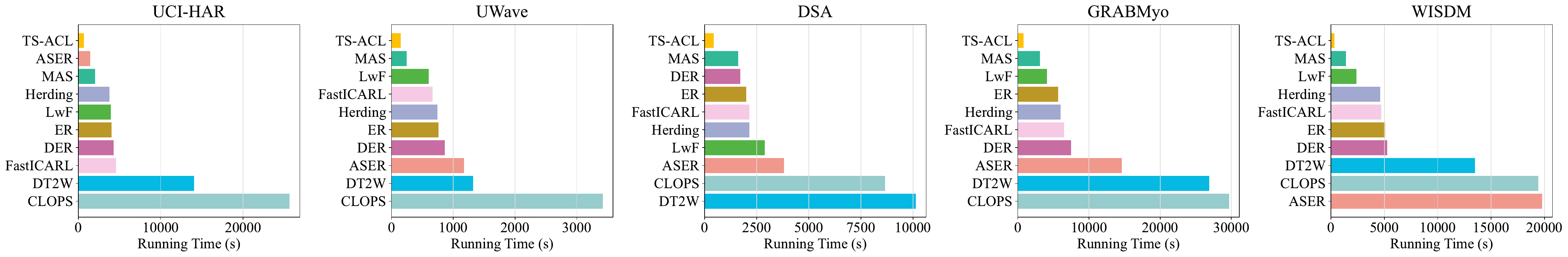} 
        \caption{Running time comparison (in seconds) across different methods on five benchmark datasets. TS-ACL consistently achieves the fastest training time across all datasets, demonstrating its superior computational efficiency. This significant reduction in computational time is attributed to TS-ACL's analytical learning approach, which requires only a single update per task after the initial training.}
        \label{fig:time_comparison}
    \end{figure*}

    \begin{table*}[htbp]
        \centering
        \caption{"Naive" refers to the baseline method without any incremental learning strategies, while "Offline" indicates methods that are trained on the entire dataset at once. For the best-performing method, we use \textbf{bold} formatting. The average value and standard deviation of each metric are reported based on the results of 5 runs.}
        \hfill

        \resizebox{\textwidth}{!}{
        \begin{NiceTabular}{c|c|cc|ccc|ccccccc|cc}
            \toprule
                         Dataset & Metric & Naive & Offline & LwF & MAS & DT$^2$W & GR& ER & DER & Herding & ASER & CLOPS & FastICARL & TS-ACL& TS-ACL-E \\ \midrule
            \multirow{2}{*}{UCI-HAR} & $\mathcal{A}_T\uparrow$ & 36.44{\tiny $\pm$10.35} & 92.31{\tiny $\pm$0.82} & 47.40{\tiny $\pm$14.04} & 59.53{\tiny $\pm$15.90} & 80.15{\tiny $\pm$6.11} & 80.04{\tiny $\pm$7.69}& 89.53{\tiny $\pm$2.41} & 90.75{\tiny $\pm$1.90} & 89.95{\tiny $\pm$2.51} & 89.82{\tiny $\pm$1.43} & 89.64{\tiny $\pm$1.50} & 85.43{\tiny $\pm$3.74} & \textbf{90.89}{\tiny $\pm$1.77} & 88.67{\tiny $\pm$5.0} \\
                           ~ & $\mathcal{F}_T\downarrow$ & 92.25{\tiny $\pm$13.25} & N.A & 74.04{\tiny $\pm$22.17} & 52.47{\tiny $\pm$27.00} & 16.55{\tiny $\pm$9.33} & 25.75{\tiny $\pm$15.74}  & 9.53{\tiny $\pm$6.54} & 8.58{\tiny $\pm$5.97} & 9.96{\tiny $\pm$7.24} & 10.20{\tiny $\pm$5.57} & 8.98{\tiny $\pm$4.82} & 17.51{\tiny $\pm$8.25} & \textbf{4.99}{\tiny $\pm$3.12} & 5.17{\tiny $\pm$2.82} \\
                           \midrule
    
            \multirow{2}{*}{UWave} & $\mathcal{A}_T\uparrow$ & 24.85{\tiny $\pm$0.12} & 96.39{\tiny $\pm$0.22} & 29.09{\tiny $\pm$5.34} & 40.74{\tiny $\pm$9.29} & 55.09{\tiny $\pm$9.27} & 85.77{\tiny $\pm$3.76} & 78.89{\tiny $\pm$4.27} & 77.74{\tiny $\pm$6.51} & 85.42{\tiny $\pm$1.89} & 77.89{\tiny $\pm$5.26} & 73.79{\tiny $\pm$3.39} & 79.01{\tiny $\pm$0.98} & 93.56{\tiny $\pm$3.36} & \textbf{95.11}{\tiny $\pm$0.5} \\
                           ~ & $\mathcal{F}_T\downarrow$ & 98.15{\tiny $\pm$1.4} & N.A & 73.47{\tiny $\pm$25.06} & 65.01{\tiny $\pm$15.14} & 40.28{\tiny $\pm$18.04} & 15.37{\tiny $\pm$4.38} & 25.87{\tiny $\pm$5.68} & 27.31{\tiny $\pm$7.67} & 16.51{\tiny $\pm$1.82} & 27.14{\tiny $\pm$6.33} & 32.12{\tiny $\pm$2.95} & 25.60{\tiny $\pm$2.24} &  3.47{\tiny $\pm$2.13} & \textbf{2.56}{\tiny $\pm$1.10} \\ \midrule
    
            \multirow{2}{*}{DSA} & $\mathcal{A}_T\uparrow$ & 19.81{\tiny $\pm$4.12} & 99.53{\tiny $\pm$0.76} & 17.01{\tiny $\pm$4.33} & 35.75{\tiny $\pm$6.35} & 19.06{\tiny $\pm$4.11} & 69.50{\tiny $\pm$7.26} & 97.24{\tiny $\pm$1.43} & 98.01{\tiny $\pm$0.69} & 97.75{\tiny $\pm$1.36} & 95.97{\tiny $\pm$6.32} & 89.65{\tiny $\pm$5.51} & 91.39{\tiny $\pm$6.16} & 98.00{\tiny $\pm$3.82}& \textbf{99.30}{\tiny $\pm$0.88} \\
                           ~ & $\mathcal{F}_T\downarrow$ & 96.23{\tiny $\pm$4.95} & N.A & 87.93{\tiny $\pm$15.21} & 66.17{\tiny $\pm$14.57} & 96.85{\tiny $\pm$4.81} & 36.43{\tiny $\pm$8.64} & 3.25{\tiny $\pm$1.78} & 2.28{\tiny $\pm$0.82} & 2.62{\tiny $\pm$1.59} & 4.73{\tiny $\pm$7.55} & 12.30{\tiny $\pm$6.64} & 10.28{\tiny $\pm$7.44} & 1.22{\tiny $\pm$2.44} & \textbf{0.48}{\tiny $\pm$0.94}\\ \midrule
    
            \multirow{2}{*}{GRABMyo} & $\mathcal{A}_T\uparrow$ & 19.46{\tiny $\pm$0.34} & 93.83{\tiny $\pm$0.87} & 19.42{\tiny $\pm$0.32} & 18.15{\tiny $\pm$1.57} & 20.09{\tiny $\pm$7.28} & 20.56{\tiny $\pm$1.37} & 61.16{\tiny $\pm$3.30} & 63.78{\tiny $\pm$3.96} & 60.07{\tiny $\pm$3.69} & 57.90{\tiny $\pm$4.79} & 52.05{\tiny $\pm$5.11} & 52.84{\tiny $\pm$3.49} & 68.92{\tiny $\pm$3.53}  & \textbf{75.27}{\tiny $\pm$2.65} \\
                           ~ & $\mathcal{F}_T\downarrow$ & 94.17{\tiny $\pm$3.34} & N.A & 93.25{\tiny $\pm$6.64} & 88.34{\tiny $\pm$7.52} & 22.57{\tiny $\pm$9.52} & 94.44{\tiny $\pm$3.13}  & 40.87{\tiny $\pm$3.59} & 37.01{\tiny $\pm$3.68} & 42.46{\tiny $\pm$4.02} & 45.49{\tiny $\pm$5.48} & 52.22{\tiny $\pm$5.54} & 52.46{\tiny $\pm$4.19} & 12.74{\tiny $\pm$1.40}  & \textbf{10.28}{\tiny $\pm$1.32} \\
                           \midrule
    
            \multirow{2}{*}{WISDM} & $\mathcal{A}_T\uparrow$ & 14.60{\tiny $\pm$2.25} & 88.60{\tiny $\pm$1.94} & 13.83{\tiny $\pm$3.51} & 19.25{\tiny $\pm$6.66} & 15.59{\tiny $\pm$7.92} & 24.33{\tiny $\pm$8.12} & 66.88{\tiny $\pm$7.36} & 67.14{\tiny $\pm$5.37} & 69.67{\tiny $\pm$3.98} & 51.48{\tiny $\pm$15.85} & 44.00{\tiny $\pm$8.11} & 44.87{\tiny $\pm$3.67} & 86.59{\tiny $\pm$1.42} & \textbf{87.11}{\tiny $\pm$1.23} \\ 
                           ~ & $\mathcal{F}_T\downarrow$ & 88.47{\tiny $\pm$10.58} & N.A & 83.53{\tiny $\pm$17.51} & 74.65{\tiny $\pm$12.06} & 37.50{\tiny $\pm$14.48} & 85.78{\tiny $\pm$10.00} & 33.80{\tiny $\pm$7.90} & 33.38{\tiny $\pm$6.67} & 30.04{\tiny $\pm$3.98} & 53.04{\tiny $\pm$18.68} & 60.74{\tiny $\pm$10.12} & 52.79{\tiny $\pm$6.37} & 5.47{\tiny $\pm$0.51} & \textbf{5.52}{\tiny $\pm$1.01} \\ 
                           \midrule
        \end{NiceTabular}}
        \label{tbl:main_results_LN}

        \label{tbl:main}
    
    \end{table*}

\noindent\textbf{Comparison Baselines.}
For exemplar-free methods, we select 3 classical baselines: \textit{LwF} \cite{LwF}, \textit{MAS} \cite{MAS}, and \textit{DT$^2$W} \cite{DT2W}. For exemplar-based methods, we choose 7 up-to-date baselines: \textit{GR} \cite{GR}, \textit{ER} \cite{ER}, \textit{ICARL} \cite{ICARL}, \textit{DER} \cite{DER}, \textit{ASER} \cite{ASER}, \textit{CLPOS} \cite{CLPOS} and \textit{FASTICARL} \cite{FASTICARL}. 

\noindent\textbf{Evaluation Metrics.}
To quantitatively assess the performance of CIL methods, we utilize two widely adopted metrics: \textit{Average accuracy} and \textit{Forgetting}. \textit{Average accuracy} is calculated by averaging the accuracy of all previously encountered tasks, including the current task after learning the current task $t$. It is defined as $\mathcal{A}_t = \frac{1}{t} \sum_{i=1}^{t} \mathcal{A}_{t,i}$, where $\mathcal{A}_{t,i}$ represents the accuracy on task $i$ after learning task $t$. \textit{Forgetting} is measured to capture how much performance degrades on previous tasks after learning a new task $t$. It is computed as $\mathcal{F}_t = \frac{1}{t-1} \sum_{i=1}^{t-1} \left( \max_{j \in \{1, \dots, t-1\}} \mathcal{A}_{j,i} - \mathcal{A}_{t,i} \right)$. At task $t$, the forgetting on task $i$ is defined as the maximum difference between the highest accuracy previously achieved on task $i$ and the accuracy on task $i$ after learning task $t$.

    \noindent\textbf{Selected encoder.}
We adopt the same 1D-CNN encoder network structure as in \cite{qiao2024class}. The network consists of four convolutional blocks, each containing a 1D convolutional layer, a normalization layer, a max-pooling layer, and a dropout layer. For the GR generator, we choose TimeVAE \cite{desai2021timevae}, where both the encoder and decoder follow a four-layer Conv1D and ConvTranspose1D structure.

    \noindent\textbf{Implementation Details.} We conduct five experiments, each using different class orders and random seeds. Hyperparameters are tuned on the validation stream for each experimental run. For comparative methods, we train models using the Adam optimizer (learning rate = 0.001, batch size = 64) for 100 epochs, applying early stopping to both initial and subsequent tasks. The learning rate scheduler is treated as a tunable hyperparameter, and early stopping is employed to prevent overfitting. We set patience values to 20 for ER-based and GR-based methods and to 5 for other approaches. To facilitate early stopping, we create a validation set by splitting the training data into a 1:9 ratio and monitor validation loss.  We evaluate three learning rate scheduling strategies: Step10, Step15, and OneCycleLR. In Step10 and Step15, the learning rate is reduced by a factor of 0.1 at the 10th and 15th epochs, respectively. Dropout rates vary across datasets: 0 for UCI-HAR and UWave, and 0.3 for DSA, GRABMyo, and WISDM.  For TS-ACL, the training process for the first task is identical to other methods. After the first task, TS-ACL uses a ridge regression method, requiring only one epoch, and adds an RHL layer after the encoder while keeping the encoder unchanged. We preprocess input data using a non-trainable normalization layer before the encoder, standardizing inputs on a per-sample basis. Normalization methods differ by dataset: layer normalization for UCI-HAR, DSA, and GRABMyo; instance normalization for UWave; and no normalization for WISDM.  For exemplar-based methods, we fix the memory buffer size at 5\% of the training size in the experiment task stream. For TS-ACL, we perform a grid search for $\gamma$ over $\{1, 10, 100\}$, set the feature expansion size to 8000, and use $n=5$ RHL layers.

    \subsection{Main Results}

    \noindent\textbf{Stability.}  
    The forgetting metric results in Table~\ref{tbl:main} highlight TS-ACL’s remarkable stability. TS-ACL achieves the lowest forgetting rates across all datasets: 4.99\% on UCI-HAR, 3.47\% on UWave, 1.22\% on DSA, 12.74\% on GRABMyo, and 5.47\% on WISDM. These forgetting rates are substantially lower than those of both exemplar-free and exemplar-based methods. In comparison, the best-performing baseline methods exhibit significantly higher forgetting rates: DT$^2$W at 16.55\% on UCI-HAR, GR at 15.37\% on UWave, DER at 2.28\% on DSA, DT$^2$W at 22.57\% on GRABMyo, and Herding at 30.04\% on WISDM.

    \noindent\textbf{Plasticity.}  
    The accuracy results in Table~\ref{tbl:main} demonstrate TS-ACL's exceptional plasticity. TS-ACL achieves outstanding accuracy across all datasets: 90.89\% on UCI-HAR, 93.56\% on UWave, 98.00\% on DSA, 68.92\% on GRABMyo, and 86.59\% on WISDM. These results significantly outperform both exemplar-free and exemplar-based methods. In comparison, the second-best exemplar-free method (DT$^2$W) shows substantially lower accuracy: 80.15\% on UCI-HAR, 55.09\% on UWave, 19.06\% on DSA, 20.09\% on GRABMyo, and 15.59\% on WISDM. The performance gaps are remarkable: 10.74\% on UCI-HAR, 38.47\% on UWave, 78.94\% on DSA, 48.83\% on GRABMyo, and 71.00\% on WISDM. Even more impressively, TS-ACL-E surpasses all exemplar-based methods, further confirming its superior plasticity.

    \noindent\textbf{Robustness to Class Order.}
    We conducted five experiments with different random seeds and reported the average performance along with the standard deviation. Our method exhibited low variance across all experiments while achieving performance nearly identical to joint training on the four datasets (UCI-HAR, UWave, DSA, and WISDM). This notable phenomenon confirms the robustness and effectiveness of our approach. The experimental results demonstrate that TS-ACL can adapt to various learning scenarios, serving as a reliable and versatile solution for TSCIL.

\begin{table*}[htbp]
    \centering
    \caption{Ablation study results showing the impact of different components. VR (Variance Ratio) is defined as the ratio of inter-class variance to intra-class variance, representing the separability between different categories. The best performance for each dataset and metric is highlighted in \textbf{bold}.}
    \label{tab:ablation}
    \begin{tabular}{lcc|cc|cc|cc|cc}
    \toprule
    \multirow{2}{*}{Components} & \multicolumn{2}{c|}{UCI-HAR} & \multicolumn{2}{c|}{UWave} & \multicolumn{2}{c|}{DSA} & \multicolumn{2}{c|}{GRABMyo} & \multicolumn{2}{c}{WISDM} \\
    \cline{2-11}
     & $\mathcal{A}_T\uparrow$  & $\text{VR}\uparrow$ & $\mathcal{A}_T\uparrow$  & $\text{VR}\uparrow$ & $\mathcal{A}_T\uparrow$  & $\text{VR}\uparrow$ & $\mathcal{A}_T\uparrow$  & $\text{VR}\uparrow$ & $\mathcal{A}_T\uparrow$  & $\text{VR}\uparrow$ \\
    \midrule
    Deep                & 86.18 & 0.59 & 87.54 & 0.85 & 91.65 & 0.26 & 40.85 & 6.83 & 48.98 & 6.98 \\
    Deep+RHL         & 87.17 & 0.70 & 89.86 & 1.02 & 96.44 & 0.32 & 56.43 & 8.35 & 82.32 & 7.71 \\
    Deep+RHL+Fusion  & \textbf{90.89} & \textbf{0.83} & \textbf{93.56} & \textbf{1.71} & \textbf{98.00} & \textbf{0.40} & \textbf{68.92} & \textbf{11.37} & \textbf{86.59} & \textbf{10.41} \\
    \bottomrule
    \end{tabular}
\end{table*}
    
    \noindent\textbf{Why TS-ACL Has Stability and Plasticity?}
    The remarkable stability and plasticity of TS-ACL stem from its theoretical guarantee of achieving results equivalent to joint learning in incremental scenarios, as discussed in Sec \ref{AL-based class incremental learning}. Specifically, by transforming each update into a gradient-free analytical learning process with a closed-form solution, TS-ACL fundamentally addresses the catastrophic forgetting issue prevalent in gradient-based methods. Meanwhile, the joint learning effect addresses intra-class variations in TSCIL \cite{qiao2024class} by learning a global distribution. The theoretical foundation ensures that TS-ACL can maintain high accuracy while minimizing forgetting. However, freezing the encoder would limit the model’s plasticity, so we introduced RHL and multi-scale feature strategies to enhance plasticity, thereby excelling in both stability and plasticity metrics.
    
    \noindent\textbf{Computational Efficiency.}
    As shown in Figure~\ref{fig:time_comparison}, TS-ACL demonstrates significant computational efficiency advantages compared to other methods. The runtime comparison across five datasets indicates that TS-ACL consistently requires far less computational time than both exemplar-based and exemplar-free methods. This efficiency can be attributed to the analytical learning approach adopted by TS-ACL, which eliminates the need for iterative gradient updates. Specifically, as analyzed in Sec \ref{efficiency}, our method updates the classification head parameters solely through matrix multiplication and matrix inversion, which can be efficiently parallelized on GPU. Moreover, while other methods require multiple rounds of training for each new task, TS-ACL completes training with just a single analytical update, significantly reducing training time. This computational advantage, combined with its outstanding accuracy and extremely low forgetting rate, makes TS-ACL particularly suitable for real-world applications, especially in resource-constrained environments (e.g., edge computing).

    \subsection{Ablation Studies}
    \noindent\textbf{Effect of  Embedding Strategy.}
    We conduct ablation studies to analyze the effectiveness of each embedding strategy in TS-ACL. Table~\ref{tab:ablation} shows the performance with different combinations of three key components: last layer feature  (Deep), RHL layer, and  multi-scale feature fusion (Fusion). The base model with only  last layer feature achieves the lowest performance (e.g., 86.18\% on UCI-HAR). Adding RHL layer brings consistent improvements across all datasets (e.g., +15.58\% on GRABMyo), with increased variance ratios indicating enhanced feature stability. The complete model with feature fusion further boosts performance significantly, achieving the best results on all datasets (e.g., 98.00\% on DSA) with substantially improved feature discrimination (e.g., variance ratio from 8.35 to 11.37 on GRABMyo). The progressive improvements in both accuracy and variance ratios validate the effectiveness of our design.

    \noindent\textbf{Effect of Ensemble Learning.}
    We also evaluated the effectiveness of the ensemble approach (TS-ACL-E) across different datasets. As shown in Table~\ref{tbl:main}, ensemble learning further improved performance on four datasets: UWave (+1.55\%), DSA (+1.30\%), GRABMyo (+6.35\%), and WISDM (+0.52\%). By aggregating multiple models with different random initializations, the ensemble method better captures diverse temporal patterns within the same class. We observed a slight performance decrease only on UCI-HAR (-2.22\%), likely due to the dataset’s relatively simple structure, where ensemble diversity may introduce unnecessary complexity. Overall, these results demonstrate that our ensemble strategy effectively enhances robustness against intra-class variations.

    \section{Discussion}
    While the TS-ACL  has demonstrated promising results in addressing class-incremental learning  for time series, there are specific conditions where its effectiveness might be limited. This section outlines potential failure modes and challenges in real-world applications.

\subsection{Scalability Challenges in Large Datasets}
While TS-ACL demonstrates excellent performance across the tested datasets, its scalability to extremely large datasets may face challenges. The method relies on a pre-trained encoder that is frozen after the initial task. For very complex or diverse data distributions, this frozen encoder might not capture sufficiently rich features to support effective classification across all future tasks.  Nevertheless, this limitation could be mitigated by leveraging more robust pre-trained backbone networks, such as Timer \cite{liu2024timergenerativepretrainedtransformers}, MOIRAI \cite{woo2024unifiedtraininguniversaltime}, and Moment \cite{goswami2024momentfamilyopentimeseries}, which have been pre-trained on large-scale time series datasets.

\subsection{Application to Real-Time Systems}
TS-ACL is particularly well-suited for real-time systems and edge computing scenarios due to several advantages. First, its analytical learning approach requires only a single update per task, significantly reducing computational demands compared to gradient-based methods that require multiple epochs as demonstrated in our experiments. Second, the method's minimal memory footprint (no need to store historical data) makes it ideal for resource-constrained devices. Third, the deterministic nature of the analytical solution ensures consistent performance without the variability often seen in stochastic gradient-based approaches. These characteristics make TS-ACL an excellent choice for applications such as wearable health monitoring, industrial IoT systems, and autonomous vehicles, where continuous learning from streaming time series data must occur with limited computational resources and strict latency requirements.
    
\section{Conclusion}
The proposed TS-ACL method tackles the core challenges of catastrophic forgetting and intra-class variations in time series class-incremental learning (TSCIL). It achieves this by employing a gradient-free, closed-form solution learning strategy integrated with multi-scale feature fusion and high-dimensional mapping techniques. Notably, our approach delivers superior performance without requiring storage of past data (exemplar-free). Experimental evaluations on five benchmark datasets confirm that TS-ACL surpasses both exemplar-based and exemplar-free methods, maintaining minimal forgetting while achieving rapid training. Both theoretical analysis and empirical results robustly validate the superiority of our method. Its computational efficiency and storage-free nature establish TS-ACL as a highly effective and practical continual learning solution for time series data in real-world scenarios.


\begingroup
\small

\bibliographystyle{IEEEtran}
\bibliography{ref}
\endgroup

\onecolumn

\appendices

\section{Proof of Theorem}
\label{app:proof}

\noindent\textbf{Theorem 1:}
The weights $\mathbf{\Phi}$, recursively obtained by
\begin{equation}
\begin{aligned}
\hat{\mathbf{\Phi}}_t = \left[ \hat{\mathbf{\Phi}}_{t-1} - \mathbf{\Psi}_{t} \mathbf{U}^{\text{E}}_t \mathbf{U}^{(\text{E})\top}_t \hat{\mathbf{\Phi}}_{t-1} \quad \mathbf{\Psi}_t \mathbf{U}^{(\text{E})\top}_t \mathbf{V}^{\text{train}}_t \right]
\end{aligned}
\end{equation}
are equivalent to those obtained from Eq. (9) for task $ t $. The matrix $ \mathbf{\Psi}_t $ can also be recursively updated by
\begin{equation}
\begin{aligned}
\mathbf{\Psi}_t = \mathbf{\Psi}_{t-1} - \mathbf{\Psi}_{t-1} \mathbf{U}^{\text{E}}_t \left( \mathbf{I} + \mathbf{U}^{(\text{E})\top}_t \mathbf{\Psi}_{t-1} \mathbf{U}^{\text{E}}_t \right)^{-1} \mathbf{U}^{(\text{E})\top}_t \mathbf{\Psi}_{t-1}.
\end{aligned}
\end{equation}


\noindent\textit{\textbf{Proof.}} 
For task $ t $, the weights $\mathbf{\Phi}$ obtained from Eq. (9) can be expressed as: 

\begin{equation}
\label{joint_solution}
    \hat{\mathbf{\Phi}}_{t} = 
    \left(\mathbf{U}^{(\text{E})\top}_{1:t} \mathbf{U}_{1:t}^{(\text{E})} + \gamma \mathbf{I}\right)^{-1} \mathbf{U}^{(\text{E})\top}_{1:t} \mathbf{V}_{1:t}.
\end{equation}

\noindent By decoupling the $t$-th task from previous tasks, $\hat{\mathbf{\Phi}}_{t}$ can be written as:

\begin{equation}
\label{eqn:decouple}
\begin{aligned}
\hat{\mathbf{\Phi}}_{t} 
&= \left(
\begin{bmatrix}
    {\mathbf{U}^{(\text{E})\top}_{1:t-1}} & {\mathbf{U}^{(\text{E})\top}_t}
\end{bmatrix}
\begin{bmatrix}
    {\left(\mathbf{U}^{(\text{E})}_{1:t-1}\right)} \\
    {\left(\mathbf{U}^{\text{E}}_t\right)}
\end{bmatrix}
 + \gamma \mathbf{I}\right)^{-1}
 \begin{bmatrix}
    {\mathbf{U}^{(\text{E})\top}_{1:t-1}}  & {\mathbf{U}^{(\text{E})\top}_t}
\end{bmatrix}
\begin{bmatrix}
    \mathbf{V}_{1:t-1}^\text{train} & \mathbf{0} \\
    \mathbf{0} & \mathbf{V}^{\text{train}}_t
\end{bmatrix} \\
            &=
\left({\mathbf{U}^{(\text{E})\top}_{1:t-1}}      
      {\mathbf{U}^{(\text{E})}_{1:t-1}}  
      + \gamma \mathbf{I}
      +{\mathbf{U}^{(\text{E})\top}_t}{\mathbf{U}^{\text{E}}_t}\right)^{-1}
\begin{bmatrix}
    {\mathbf{U}^{(\text{E})\top}_{1:t-1}} \mathbf{V}_{1:t-1}^\text{train} & {\mathbf{U}^{(\text{E})\top}_t}
    \mathbf{V}^{\text{train}}_t\\
\end{bmatrix}.
\end{aligned}
\end{equation}

\noindent 
According to Eq. (10), the aggregated temporal inverse correlation matrix for task $ t $ can be expressed as
\begin{equation}
    \label{eqn:M_p}
    \mathbf{\Psi}_{t} = 
    \left(\mathbf{U}^{(\text{E})\top}_{1:t} \mathbf{U}_{1:t}^{(\text{E})} + \gamma \mathbf{I}\right)^{-1} = 
    \left({\mathbf{U}^{(\text{E})\top}_{1:t-1}}      
      {\mathbf{U}^{(\text{E})}_{1:t-1}}  
      + \gamma \mathbf{I}
      +{\mathbf{U}^{(\text{E})\top}_t}{\mathbf{U}^{\text{E}}_t}\right)^{-1} =
      \left( {\mathbf{\Psi}_{t-1}}^{-1}
      +{\mathbf{U}^{(\text{E})\top}_t}{\mathbf{U}^{\text{E}}_t}\right)^{-1}.
\end{equation}

\noindent 
By applying the Woodbury matrix identity, where  $\left(\mathbf{A} + \mathbf{UCV}\right)^{-1} = \mathbf{A}^{-1} - \mathbf{A}^{-1} \mathbf{U} \left(\mathbf{C}^{-1} + \mathbf{V} \mathbf{A}^{-1} \mathbf{U}\right)^{-1} \mathbf{V} \mathbf{A}^{-1}$ (which will be proved in Appendix \ref{woodbury}), and treating $\mathbf{\Psi}_{t-1}$ as $\mathbf{A}^{-1}$, the memory matrix for task $ t $ can be further represented as:
\begin{equation}
\label{eqn:memory_p}
\mathbf{\Psi}_{t} =
 \mathbf{\Psi}_{t-1} - \mathbf{\Psi}_{t-1} \mathbf{U}^{(\text{E})\top}_t \left(\mathbf{I} + \mathbf{U}^{\text{E}}_t \mathbf{\Psi}_{t-1} \mathbf{U}^{(\text{E})\top}_t\right)^{-1} \mathbf{U}^{\text{E}}_t \mathbf{\Psi}_{t-1}.
\end{equation}

\noindent Thus, the weights $\hat{\mathbf{\Phi}}_{t}$ can be derived as
\begin{equation}
\label{eqn:w_n_derive}
    \hat{\mathbf{\Phi}}_{t} = 
    \begin{bmatrix}
    \mathbf{\Psi}_{t} {\mathbf{U}^{(\text{E})\top}_{1:t-1}} \mathbf{V}_{1:t-1}^\text{train} &
    \mathbf{\Psi}_{t} {\mathbf{U}^{(\text{E})\top}_t} \mathbf{V}^{\text{train}}_t
    \end{bmatrix}.
\end{equation}

\noindent Denote the left submatrix $\mathbf{\Psi}_{t} \mathbf{U}^{(\text{E})\top}_{1:t-1} \mathbf{V}_{1:t-1}^\text{train}$ as $\mathbf{H}$. By substituting Eq. (\ref{eqn:memory_p}) into Eq. (\ref{eqn:w_n_derive}),
\begin{equation}
    \mathbf{H} = 
    \hat{\mathbf{\Phi}}_{t-1} - \mathbf{\Psi}_{t-1} \mathbf{U}^{(\text{E})\top}_t \left(\mathbf{I} + \mathbf{U}^{\text{E}}_t \mathbf{\Psi}_{t-1} \mathbf{U}^{(\text{E})\top}_t\right)^{-1}
   \mathbf{U}^{\text{E}}_t \hat{\mathbf{\Phi}}_{t-1}.
\end{equation}
     
\noindent Let $\mathbf{P} = \mathbf{U}^{\text{E}}_t \mathbf{\Psi}_{t-1} \mathbf{U}^{(\text{E})\top}_t$. 
Based on the identity $\left(\mathbf{I} + \mathbf{P}\right)^{-1}=\mathbf{I} - \left(\mathbf{I} + \mathbf{P}\right)^{-1}\mathbf{P}$, $\mathbf{H}$ can be further derived as:
\begin{equation}
\begin{aligned}
        \mathbf{H} &= \hat{\mathbf{\Phi}}_{t-1} - \mathbf{\Psi}_{t-1} \mathbf{U}^{(\text{E})\top}_t [\mathbf{I} - \left(\mathbf{I} + \mathbf{U}^{\text{E}}_t \mathbf{\Psi}_{t-1} \mathbf{U}^{(\text{E})\top}_t\right)^{-1}\mathbf{P}]
   \mathbf{U}^{\text{E}}_t \hat{\mathbf{\Phi}}_{t-1} \\
    &= \hat{\mathbf{\Phi}}_{t-1} - [\mathbf{\Psi}_{t-1} \mathbf{U}^{(\text{E})\top}_t  + \mathbf{\Psi}_{t-1} \mathbf{U}^{(\text{E})\top}_t \left(\mathbf{I} + \mathbf{U}^{\text{E}}_t \mathbf{\Psi}_{t-1} \mathbf{U}^{(\text{E})\top}_t\right)^{-1}\mathbf{P}] \mathbf{U}^{\text{E}}_t \hat{\mathbf{\Phi}}_{t-1} \\
    &= \hat{\mathbf{\Phi}}_{t-1} - [\mathbf{\Psi}_{t-1} \mathbf{U}^{(\text{E})\top}_t  + \mathbf{\Psi}_{t-1} \mathbf{U}^{(\text{E})\top}_t \left(\mathbf{I} + \mathbf{U}^{\text{E}}_t \mathbf{\Psi}_{t-1} \mathbf{U}^{(\text{E})\top}_t\right)^{-1} \mathbf{U}^{\text{E}}_t \mathbf{\Psi}_{t-1} \mathbf{U}^{(\text{E})\top}_t] \mathbf{U}^{\text{E}}_t \hat{\mathbf{\Phi}}_{t-1} \\
    &= \hat{\mathbf{\Phi}}_{t-1} - [\mathbf{\Psi}_{t-1} + \mathbf{\Psi}_{t-1} \mathbf{U}^{(\text{E})\top}_t \left(\mathbf{I} + \mathbf{U}^{\text{E}}_t \mathbf{\Psi}_{t-1} \mathbf{U}^{(\text{E})\top}_t\right)^{-1} \mathbf{U}^{\text{E}}_t \mathbf{\Psi}_{t-1}] \mathbf{U}^{(\text{E})\top}_t \mathbf{U}^{\text{E}}_t \hat{\mathbf{\Phi}}_{t-1} \\
    &= \hat{\mathbf{\Phi}}_{t-1} - \mathbf{\Psi}_{t}\mathbf{U}^{(\text{E})\top}_t
    \mathbf{U}^{\text{E}}_t \hat{\mathbf{\Phi}}_{t-1}. 
\end{aligned}
\end{equation}

\noindent Thus, 
\begin{equation}
    \hat{\mathbf{\Phi}}_{t} = 
\begin{bmatrix}
    {\hat{\mathbf{\Phi}}_{t-1} - \mathbf{\Psi}_{t} \mathbf{U}^{(\text{E})\top}_t \mathbf{U}^{\text{E}}_t \hat{\mathbf{\Phi}}_{t-1}} & \mathbf{\Psi}_{t} \mathbf{U}^{(\text{E})\top}_t \mathbf{V}^{\text{train}}_t \\
\end{bmatrix}.
\end{equation}

\noindent Theorem 1 is thus proved.
$\hfill\blacksquare$

\section{Proof of Woodbury matrix identity}
\label{woodbury}
\noindent\textbf{Woodbury matrix identity:} For matrices $\mathbf A\in \mathbb{R}^{n\times n}$, $\mathbf U \in \mathbb{R}^{n\times m}$, $\mathbf C\in \mathbb{R}^{m\times m}$, and $\mathbf V\in \mathbb{R}^{n\times n}$, if $\mathbf A$ and $\mathbf C$ are both reversible, we can derive the following expression:
    \begin{equation}
    (\mathbf A+ \mathbf U \mathbf C \mathbf V)^{-1}=\mathbf{A}^{-1}-\mathbf{A}^{-1}\mathbf{U}(\mathbf{C}^{-1}+\mathbf{V}\mathbf{A}^{-1}\mathbf{U})^{-1}\mathbf{V}\mathbf{A}^{-1}.
\end{equation}

\noindent\textit{\textbf{Proof.}} 
First, by multiplying $ \mathbf{A} + \mathbf{U}\mathbf{C}\mathbf{V} $ on the right by $ \mathbf{A}^{-1} $, we have:
\begin{equation}
(\mathbf{A} + \mathbf{U}\mathbf{C}\mathbf{V})\mathbf{A}^{-1} = \mathbf{I} + \mathbf{U}\mathbf{C}\mathbf{V}\mathbf{A}^{-1}.
\end{equation}

\noindent Subsequently, we multiply it on the right by $ \mathbf{U} $:
\begin{equation}
(\mathbf{I} + \mathbf{U}\mathbf{C}\mathbf{V}\mathbf{A}^{-1})\mathbf{U} = \mathbf{U} + \mathbf{U}\mathbf{C}\mathbf{V}\mathbf{A}^{-1}\mathbf{U}.
\end{equation}

\noindent Since $ \mathbf{C} $ is invertible, we have:
\begin{equation}
(\mathbf{A} + \mathbf{U}\mathbf{C}\mathbf{V})\mathbf{A}^{-1}\mathbf{U} = \mathbf{U}\mathbf{C}(\mathbf{C}^{-1} + \mathbf{V}\mathbf{A}^{-1}\mathbf{U}).
\end{equation}

\noindent Since $ \mathbf{C}^{-1} + \mathbf{V}\mathbf{A}^{-1}\mathbf{U} $ is invertible, we obtain:
\begin{equation}
\mathbf{U}\mathbf{C} = (\mathbf{A} + \mathbf{U}\mathbf{C}\mathbf{V})\mathbf{A}^{-1}\mathbf{U}(\mathbf{C}^{-1} + \mathbf{V}\mathbf{A}^{-1}\mathbf{U})^{-1}.
\end{equation}

\noindent 
To connect it with the left-hand side of the Woodbury matrix identity, we transform $ \mathbf{U}\mathbf{C} $ into $ \mathbf{A} + \mathbf{U}\mathbf{C}\mathbf{V} $:
\begin{equation}
\mathbf{A} + \mathbf{U}\mathbf{C}\mathbf{V} = \mathbf{A} + (\mathbf{A} + \mathbf{U}\mathbf{C}\mathbf{V})\mathbf{A}^{-1}\mathbf{U}(\mathbf{C}^{-1} + \mathbf{V}\mathbf{A}^{-1}\mathbf{U})^{-1}\mathbf{V}.
\end{equation}

\noindent By multiplying it on the right by $ \mathbf{A}^{-1} $, we can obtain:
\begin{equation}
(\mathbf{A} + \mathbf{U}\mathbf{C}\mathbf{V})\mathbf{A}^{-1} = \mathbf{I} + (\mathbf{A} + \mathbf{U}\mathbf{C}\mathbf{V})\mathbf{A}^{-1}\mathbf{U}(\mathbf{C}^{-1} + \mathbf{V}\mathbf{A}^{-1}\mathbf{U})^{-1}\mathbf{V}\mathbf{A}^{-1}.
\end{equation}

\noindent Then we multiply on the left by $ (\mathbf{A} + \mathbf{U}\mathbf{C}\mathbf{V})^{-1} $:
\begin{equation}
\mathbf{A}^{-1} = (\mathbf{A} + \mathbf{U}\mathbf{C}\mathbf{V})^{-1} + \mathbf{A}^{-1}\mathbf{U}(\mathbf{C}^{-1} + \mathbf{V}\mathbf{A}^{-1}\mathbf{U})^{-1}\mathbf{V}\mathbf{A}^{-1}.
\end{equation}

\noindent Rearranging the terms, we can finally obtain:
\begin{equation}
(\mathbf{A} + \mathbf{U}\mathbf{C}\mathbf{V})^{-1} = \mathbf{A}^{-1} - \mathbf{A}^{-1}\mathbf{U}(\mathbf{C}^{-1} + \mathbf{V}\mathbf{A}^{-1}\mathbf{U})^{-1}\mathbf{V}\mathbf{A}^{-1}.
\end{equation}

\noindent This completes the proof.
$\hfill\blacksquare$

\end{document}